%% file: learn-discrete.tex
\documentclass[11pt]{article}
\usepackage{fullpage}
\usepackage{mathtools}
\usepackage{amsfonts}
\usepackage{amsthm}
\usepackage{hyperref}
\usepackage{xcolor}
\usepackage{enumerate}
\usepackage{indentfirst}
\usepackage{authblk}
\usepackage[ruled,vlined,linesnumbered]{algorithm2e}
\hypersetup{
	colorlinks	= true,
	urlcolor	= orange,
	citecolor	= blue,
	linkcolor	= red
}

\newcommand{\eps}{\epsilon}

\newcommand{\argmin}{\operatorname*{argmin}}
\newcommand{\arrange}[2]{A_{#1}^{#2}}
\newcommand{\biest}{\textsf{BinomialEst}}
\newcommand{\cnt}{\mathrm{cnt}}
\newcommand{\distset}{\textsf{DistSet}}
\newcommand{\err}{\eta}

\newcommand{\ind}[1]{\mathbb{I}\left\{#1\right\}}
\newcommand{\KL}{\mathrm{KL}}
\newcommand{\norm}[1]{\left\lVert#1\right\rVert_1}

\newcommand{\poly}{\operatorname*{poly}}
\newcommand{\polylog}{\operatorname*{polylog}}
\newcommand{\real}{\mathbb{R}}
\newcommand{\tot}{\mathrm{tot}}
\newcommand{\tp}[2]{{#1}^{\otimes #2}}

\newtheorem{theorem}{Theorem}[section]
\newtheorem{lemma}[theorem]{Lemma}
\newtheorem{claim}[theorem]{Claim}

\title{
	Learning Discrete Distributions from Untrusted Batches
}
\date{}

\author[1]{Mingda Qiao}
\author[2]{Gregory Valiant\thanks{Gregory's work is supported by ONR award N00014-17-1-2562 and NSF CAREER award CCF‐1351108.}}
\affil[1]{Institute for Interdisciplinary Information Sciences, Tsinghua University\authorcr\texttt{qmd14@mails.tsinghua.edu.cn}}
\affil[2]{Computer Science Department, Stanford University\authorcr\texttt{valiant@stanford.edu}}

\begin{document}
\maketitle

\begin{abstract}
	We consider the problem of learning a discrete distribution in the presence of an $\epsilon$ fraction of malicious data sources.  Specifically, we consider the setting where there is some underlying distribution, $p$, and each data source provides a batch of $\ge k$ samples, with the guarantee that at least a $(1-\eps)$ fraction of the sources draw their samples from a distribution with total variation distance at most $\eta$ from $p$.  We make no assumptions on the data provided by the remaining $\eps$ fraction of sources--this data can even be chosen as an adversarial function of the $(1-\eps)$ fraction of ``good'' batches.  We provide two algorithms: one with runtime exponential in the support size, $n$, but polynomial in $k$, $1/\eps$ and $1/\eta$ that takes $O((n+k)/\eps^2)$ batches and recovers $p$ to error $O(\eta+\eps/\sqrt{k})$.  This recovery accuracy is information theoretically optimal, to constant factors, even given an infinite number of data sources.  Our second algorithm applies to the $\eta=0$ setting and also achieves an $O(\eps/\sqrt{k})$ recover guarantee, though it runs in $\poly((nk)^k)$ time.  This second algorithm, which approximates a certain tensor via a rank-1 tensor minimizing $\ell_1$ distance, is surprising in light of the hardness of many low-rank tensor approximation problems, and may be of independent interest.
\end{abstract}

\input{intro}
\input{bound}
\input{const_n}
\input{const_k}

\bibliographystyle{alpha}
\bibliography{learn-discrete}

\clearpage
\appendix
\input{tensor}
\input{missing}

\end{document}

%% file: intro.tex
\section{Introduction}
	Consider the following real-world problem: suppose there are millions of people texting away on their phones, and you wish to learn the distribution of words corresponding to a given mis-typed word, or the distribution of words that follow a given sequence, etc.  The challenge of this setup is twofold.  First, each person provides far too little data to accurately learn these distributions based solely on one person's data, hence a successful learning or estimation algorithm must combine data from different sources.  Second, these sources are heterogeneous---some people have wider fingers than others, and the nature of typos likely differs between people.  Further complicating this heterogeneity is the very real possibility that a small but not negligible fraction of the sources might be operated by adversarial agents whose goal is to embarrass the learning algorithm either as a form of corporate sabotage or as a publicity stunt.  In much the same way as ``Google bombing'' or ``link bombing'' was used to associate specific websites with certain terms---perhaps the most famous of which resulting in the George W. Bush biography being the top Google and Yahoo! hit when searching for the term ``miserable failure'' back in Sept.~2006---it seems likely that certain groups of people would collectively attempt to influence the auto-correct or auto-suggest responses to certain misspellings.
    
    This general problem of learning or estimation given data supplied by a large number of individuals has gained attention in the more practical communities under the keyword \emph{Federated Learning} (see e.g.~\cite{mcmahan2016communication,45648} and Google research blog post~\cite{fedLearning}), and raises a number of pertinent questions: What notions of privacy can be maintained while leveraging everyone's data? Do we train a single model or try to personalize models for each user?  
    
    In this work, we consider a basic yet fundamental problem in this space: learning a discrete distribution given access to batches of samples, where an unknown $(1-\eps)$ fraction of batches are drawn i.i.d. from distributions that each has total variation (equivalently, $\ell_1$) distance at most $\err$ from some target distribution, $p$, and we make no assumptions about the data contained in the remaining $\eps$ fraction of batches.  The data in this ``bad'' $\eps$ fraction of the batches can even be chosen adversarially  as a function of the ``good'' data.  This problem is also a natural problem in the space of ``robust statistics'', and the recent line of recent work from the theory community on estimation and learning with untrusted data (e.g.~\cite{lai2016agnostic,diakonikolas2016robust,steinhardt2016avoiding,charikar2017learning,steinhardt2017resilience,diakonikolas2017being}), and we outline these connections in Section~\ref{sec:related}.
    
    We begin by summarizing our main results and discuss the connections with related work.  We conclude the introduction with a discussion of several relevant directions for future research.

    \subsection{Summary of Results}\label{sec:res}
    
    The following theorem characterizes our tight information theoretic result for robustly learning a discrete distribution in the setting where a certain fraction of batches of data are arbitrarily corrupted:
    
    \begin{theorem}\label{thm:main}
    Let $p$ denote a distribution supported on $n$ domain elements, and fix parameters $\eps \in (0,1/900)$, $\err > 0$, failure probability $\delta \in (0, 1)$, and integer $k \ge 1$.  Suppose we have access to $m = O((n+k+\log(1/\delta))/\eps^2)$ batches of data such that at least $m(1-\eps)$ of the batches consist of $\ge k$ i.i.d. draws from some distribution with $\ell_1$ distance at most $\err$ from $p$.  There exists an algorithm that runs in time $\poly\left(2^n,k,1/\eps,1/\err,\log(1/\delta)\right)$ and, with probability at least $1-\delta$, returns a distribution $\hat{p}$ satisfying $\|p-\hat{p}\|_1 = O(\err+\eps/\sqrt{k}),$ where the ``O'' notation hides an absolute constant. 
    \end{theorem}

    The recovery guarantees of the above theorem are information theoretically optimal, even given infinite data, as the following lower bound formalizes:
    
    \begin{theorem}\label{thm:lb}
    In the setup of Theorem~\ref{thm:main} for integers $k \ge 1$ and $n \ge 2$ and parameters $\eps \in (0,1/2)$ and $\err \in [0,1/4)$, no algorithm can, with probability greater than $1/2$, return a distribution $\hat{p}$ satisfying $\|p-\hat{p}\|_1 < 2\err + \eps/\sqrt{2 k},$ even in the limit as the number of batches, $m$, tends to infinity. 
    \end{theorem}
    
To provide some intuition for these results, consider the setting where $\err=0$, namely where the $(1-\eps)$ fraction of ``good'' batches all consist of $k$ i.i.d. draws from $p$.  In this case, the above results claim an optimal recovery error of $\Theta(\epsilon / \sqrt{k})$.  Intuitively, this is due to the following ``tensorization'' property of the total variation distance, which we show in Appendix~\ref{sec:tensor}: given two distributions $p$ and $q$ with $\|p-q\|_1 = \alpha = O(1/\sqrt{k})$, the $\ell_1$ distance between their $k$th tensor products, $\|\tp{p}{k} - \tp{q}{k}\|_1$, is at least $\Theta(\alpha \sqrt{k}).$  Here, the $k$th tensor product $\tp{p}{k}$ denotes the $n^k$ sized object, indexed by a $k$-tuple $i_1,\ldots,i_k$, with $\tp{p}{k}_{i_1,\ldots,i_k}= \prod_{j=1}^k p(i_j)$.  This lower bound on the distance between the $k$th tensor products implies that, even though the adversary can alter the observed $k$-fold product distribution of $p$ by $\epsilon$ (w.r.t. the $\ell_1$ or total variation distance), any two $k$-tensors with distance $\le \eps$ that each correspond to the $k$th tensor products of distributions will correspond to distributions whose distance is only $O(\epsilon/\sqrt{k})$, motivating our claimed recovery guarantees. 
    
     The algorithm to which the guarantees of Theorem~\ref{thm:main} apply proceeds by reducing the problem at hand to $\le 2^n$ instances of the problem of learning a Bernoulli random variable in this setting of untrusted batches.  Specifically, for each of the $\le 2^n$ possible subsets of the observed domain elements, the algorithm attempts to estimate the probability mass that $p$ assigns to that set.  The algorithm then combines these $\le 2^n$ estimates to estimate $p$.  In Section~\ref{sec:future}, we discuss the possibility of an efficient variant of this algorithm that proceeds via estimating these weights for only a polynomial number of subsets.  We also note that the requirement that $\eps \le 1/900$ can be relaxed slightly at the expense of readability of the proof.
     
        Our proof of the lower bound (Theorem~\ref{thm:lb}) proceeds by providing a construction of an explicit pair of indistinguishable instances, one corresponding to a distribution $p$, one corresponding to a distribution $q$, with $\|p-q\|_1= 4\err+2\eps/\sqrt{2k}$.  Each instance consists of two parts---a distribution from which the ``good'' batches of data are drawn, and a distribution over $k$-tuples of samples from which the $\eps$ fraction of ``bad'' batches are drawn.  For the pair of instances constructed, the mixture of the good and bad batches corresponding to $p$ is identical to the mixture of the good and bad batches corresponding to $q$. 
        
\medskip

    In addition to the algorithm of Theorem~\ref{thm:main}, which runs in time exponential in the support size, we provide a second algorithm in the setting where $\err = 0$, with runtime $(nk)^{O(k)}$ that also achieves the information theoretically optimal recovery error of $O(\eps/\sqrt{k}).$
    
    \begin{theorem}\label{thm:efficient}
    As in the setup of Theorem~\ref{thm:main}, given $m=(nk)^{O(k)}\log(1/\delta)/\eps^2$ batches of samples, of which a $(1-\eps)$ fraction consist of $\ge k$ i.i.d. draws from a distribution $p$, supported on $\le n$ elements, there is an algorithm that runs in time $\poly((nk)^k,1/\eps,\log(1/\delta))$ which, with probability at least $1-\delta$, returns a distribution $\hat{p}$ such that $\|p-\hat{p}\|_1 = O(\eps/\sqrt{k}).$ 
    \end{theorem}
    
    The algorithm to which the above theorem applies proceeds by forming the $k$-tensor corresponding to the empirical distribution over the $m$ batches, with each batch regarded as a $k$-tuple.   The algorithm then efficiently computes a rank-1 $k$-tensor which approximates this empirical tensor in the (element-wise) $\ell_1$ sense.  The ability to efficiently (in time polynomial in the size of the $k$-tensor) compute this rank-1 approximation is rather surprising, and crucially leverages the structure of the empirical $k$-tensor guaranteed by the setup of our problem.   Indeed the general problem of finding the best rank-1 approximation (in either an $\ell_1$ or $\ell_2$ sense) of a $k$-tensor is NP-hard for $k \ge 3$~\cite{hillar2013most}.  Additionally, even for various ``nice'' random distributions of 3-tensors, efficient algorithms for related rank-1 approximation problems would yield efficient algorithms for recovering planted cliques of size $o(\sqrt{n})$ from $G(n,1/2)$~\cite{frieze2008new}.\footnote{Specifically, \cite{frieze2008new} shows that given an $n\times n \times n$ 3-tensor $A$ constructed from an instance of planted-clique, the ability to efficiently find a unit vector $v$ that nearly maximizes $\max_{\|v\|=1} \sum_{i,j,k} v_i v_j v_k A_{i,j,k}$ would yield an algorithm for finding cliques of size at least $O(n^{1/3}\polylog(n))$ planted in $G(n,1/2)$.}

	\subsection{Related Work}\label{sec:related}
    
    There are a number of relevant lines of related work, including work from the theoretical computer science and information theory communities on learning, estimating and testing properties of distributions or collections of distributions, the classical work on  ``robust statistics'', the recent line of work from the computer science community on robust learning and estimation with untrusted data and ``adversarial'' machine learning, and the very recent work from the more applied community on ``federated learning''.  From a technical perspective, there are also connections between our work and the body of work on tensor factorizations, and low rank matrix and tensor approximations with respect to the $\ell_1$ norm.   We briefly summarize these main lines of related work.

		\paragraph{Learning and Testing Discrete Distributions.}
		The problem of learning a discrete distribution given access to independent samples has been intensely studied over the past century (see, e.g.~\cite{gilbert1971codes,krichevsky1981performance} and the references therein).  For distributions supported on $n$ elements, given access to $k$ i.i.d. samples, minimax optimal loss rates as a function of $n$ and $k$ have been considered for various loss functions, including KL-divergence~\cite{braess2004bernstein}, $\ell_2$~\cite{kamath2015learning}, $\chi^2$ loss~\cite{kamath2015learning}.  For $\ell_1$ error (equivalently, ``total variation distance'' or ``statistical distance''), it is easy to show that the expected worst-case error is $\Theta(\sqrt{n/k})$, and the recent work~\cite{kamath2015learning} establishes the exact first-order coefficients of this loss.  Beyond this worst-case setting, there has also been significant work considering this learning problem with the goal of developing ``instance optimal'' algorithms that leverage whatever structure might be present in the given distribution (see e.g.~\cite{orlitsky2015competitive,valiant2016instance}).  
        
        In the context of work on testing distributional properties from the theoretical computer science community, the work closest to this current paper is the work of Levi, Ron, and Rubinfeld~\cite{levi2013testing}, which considers the following task:  given access to independent draws from each of $m$ distributions, distinguish whether all $m$ distributions are identical (or very close), versus the case that there is significant variation between the distributions---namely that the average distance to any single distribution is at least a constant.   We note that in this work, as in much of the distribution testing literature, the results are extremely sensitive to assumption that in the ``yes'' case, the distributions are all extremely close to one distribution.  For example, the task of distinguishing whether the distributions have average distance at most $\alpha_1$ from a single distribution, versus having average distance at least $\alpha_2$, seems to be a significantly harder problem when $\alpha_1$ is not asymptotically smaller than $\alpha_2$.\footnote{For example, testing whether two distributions supported on $\le n$ elements are identical, vs have distance at least $0.1$ requires $\Theta(n^{2/3})$ samples, whereas distinguishing whether the two distributions have distance $\le 0.1$ versus $\ge 0.9$ requires $\Theta(n/\log n)$ samples.~\cite{batu2000testing,chan2014optimal,valiant2011estimating}.}
        
        In a related vein, the recent work~\cite{TKV17} considers the setting of drawing $k$ i.i.d. samples from each of $m$ possibly heterogeneous distributions, and shows that the \emph{set} of such distributions can be accurately learned.  For example, if each distribution is a Bernoulli random variable with the $i$th distribution corresponding to some probability $p_i$, the histogram of the $p_i$'s can be recovered to error $O(1/k)$, rather than the $\Theta(1/\sqrt{k})$ that would be given by the empirical estimates.    Both~\cite{levi2013testing} and~\cite{TKV17}, however, crucially rely on the assumption that the data consists of independent draws from the $m$ distributions, and the algorithms and results do not extend to the present setting where some non-negligible fraction of the data is arbitrary/adversarial.
        
\paragraph{Robust Statistics, Learning with Unreliable Data, and Adversarial Learning.}

The problem of estimation and learning in the presence of contaminated or outlying data points has a history of study in the Statistics community, dating back to early work of Tukey~\cite{tukey1960survey} (see also the surveys~\cite{huber2robust,hampel2011robust}).  Recently, these problems have gained attention from both the applied and theoretical computer science communities.  On the applied side, the interest in ``adversarial'' machine learning stems from the very real vulnerabilities of many deployed learning systems to ``data poisoning'' attacks in which an adversary selectively alters or plants a small amount of carefully chosen training data that significantly impacts the resulting trained model (see e.g.~\cite{newsome2006paragraph,nelson2011understanding,xiao2015support}).  

On the theory side, recent work has revisited some of the classical robust statistics settings with an eye towards 1) establishing the information theoretic dependencies between the fraction of corrupted data, dimension of the problem, and achievable accuracy of returned model or estimate, and 2) developing computationally tractable algorithms that approach or achieve these information theoretic bounds.  These recent works include~\cite{lai2016agnostic,diakonikolas2016robust,diakonikolas2017being} who consider the problem of robustly estimating the mean and covariance of high-dimensional Gaussians (and other distributions such as product distributions), and~\cite{steinhardt2016avoiding} who consider a sparse ratings aggregation setting.  Other learning problems, such as robustly learning halfspaces~\cite{klivans2009learning}, robust linear regression~\cite{bhatia2015robust} and more general convex optimization~\cite{charikar2017learning} have also been considered in similar settings where some fraction of the data is drawn from a distribution of interest, and no assumptions are made about the remaining data.  This latter work and~\cite{meister2017data} also considers several models for which strong positive results can be attained even when the majority of the data is arbitrary (i.e. $\eps > 1/2).$ 

The present setting, in which data arrives in batches, with some batches corresponding to ``good'' data represents a practically relevant instance of this more general robust learning problem, and we are not aware of previous work from the theory community that explicitly considers it.\footnote{One could view each batch of $k$ samples as a single $k$-dimensional draw from a product of multinomial distributions, or as a $k$-sparse sample from an $n$-dimensional product distribution, although the results of previous work do not yield strong results for these settings.}   The recent attention on ``federated learning''~\cite{mcmahan2016communication,45648,fedLearning} focuses on a similar setting where data is presented in batches (corresponding to individual users), although the current emphasis is largely on privacy and communication concerns, as opposed to robustness.

\paragraph{Low Rank Approximations in $\ell_1$.}
The algorithm underlying Theorem~\ref{thm:efficient} efficiently computes a rank-1 approximation of the empirical $k$-tensor of data, where the recovered rank-1 tensor is close in the (element-wise) $\ell_1$ sense.  Both the problems of computing the best low-rank $\ell_1$ approximation of a \emph{matrix}, and the problems of computing the best rank-1 approximation of a $k$-tensor for $k \ge 3$ are NP-hard~\cite{gillis2015complexity,hillar2013most}.  Nevertheless, recent work has provided efficient algorithms for returning good (in a competitive-analysis sense) low-rank approximations of matrices in the $\ell_1$ norm~\cite{song2017low}, and for efficiently recovering low-rank approximations of special classes of tensors, including those with random or ``incoherent'' factors (see e.g.~\cite{kruskal1977three,ma2016polynomial}).  Still, for many tensor decomposition problems over various families of random nearly low-rank tensors---such as those obtained from instances of planted-clique~\cite{frieze2008new}---these decomposition problems remain mysterious.  
    
    \subsection{Future directions and discussion}\label{sec:future}
    
    Two concrete open directions raised by this work are 1) understanding the computational hardness of this basic learning question, and 2) extending the information theoretic and algorithmic results to the setting where a \emph{minority} of the batches of data are ``good'', which generalizes the problem of robustly learning a mixture of distributions.  Before discussing these problems, we note that there are a number of other practically relevant related questions, including the extent to which this problem can be solved while maintaining some notion of privacy for each batch/user, and considering other basic learning and estimation tasks in this batched robust setting.

    \paragraph{An Efficient Algorithm?}
	Recall that the algorithm of Theorem~\ref{thm:main} has a linear sample complexity
	(treating $\eps$ as fixed),	yet requires a runtime exponential in $n$.  The algorithm of Theorem~\ref{thm:efficient} requires more data, but runs in time $(nk)^{O(k)}$, and only applies to the clean but less practically relevant setting where $\err = 0$.   One natural question is whether there exists an algorithm that achieves both the linear sample complexity and the $(nk)^{O(k)}$ runtime.  Alternately, does there exist an algorithm with data requirements and runtime that are polynomial in all the parameters, $n,k,1/\eps,$ and  achieves the optimal recovery guarantees even in the $\err=0$ regime?  Or are there natural barriers to such efficiency or connections to other problems that we believe to be computationally intractable?
    
    One natural approach to yielding an efficient algorithm is via a variant of our first algorithm.  At a high level, our algorithm proceeds by recovering $\eps/\sqrt{k}$-accurate estimates of the probability mass of each of the $2^n$ possible subsets of the $n$ domain elements, and then recovering a distribution consistent with these estimates via a linear program.  A natural approach to improving the running time of the algorithm
	is to find an efficient separation oracle for this linear program.  Alternately, one could even imagine recovering $\eps/\sqrt{k}$-accurate estimates of the mass of a random $\poly(k,n)$ sized set of subsets of the domain and then using these to recover the distribution.  If the error in each estimate were independent, say distributed according to $N(0,\eps^2/k)$, then a polynomial (and even linear!) number of such measurements would suffice to recover the distribution to error $O(\eps/\sqrt{k}).$  On the other hand, if an adversary is allowed to choose the errors in each measurement, trivially, an exponential number would be required.  In our setting, however, the adversaries are somewhat limited in their ability to corrupt the measurements, and it is plausible that such an approach could work, perhaps both in practice and theory.

	\paragraph{The small-$\alpha$ regime.}
	Our positive results (Theorems \ref{thm:main}~and~\ref{thm:efficient}) assume that the fraction of adversarial data, $\epsilon$, is
	relatively small (at most of order $10^{-3}$).
	While this assumption can be slightly relaxed
	by carefully adjusting the constants in the proofs,
	it is also natural to ask:
	what can we do if $\epsilon$ is much larger and even close to $1$?
	This regime is examined in~\cite{charikar2017learning,meister2017data}.
	It is clearly impossible to estimate the distribution in question to a nontrivial precision in the setting where $\eps \ge 1/2$ due to the symmetry
	between the real distribution and the one chosen by the adversary.  Nevertheless, two learning frameworks were proposed for this case in~\cite{charikar2017learning}:
	the \emph{list-decodable learning} framework (first introduced
	by Balcan et al.~\cite{balcan2008discriminative}),
	in which the learning algorithm is allowed to output multiple answers to the learning problem (perhaps $1/(1-\eps)$),
	and the \emph{semi-verified} model,
	where a small amount of reliable/verified data is available.  It seems plausible that variants of our algorithms that attempt to recover a rank $1/(1-\eps)$ approximation to the empirical tensor might be successful, though the algorithm and analysis are certainly not immediate.

\subsection{Organization of Paper}
In Section~\ref{sec:bound} we establish our information theoretic lower bound, Theorem~\ref{thm:lb}, which bounds the accuracy of the recovered distribution and applies even in the infinite data setting.  Section~\ref{sec:const-n} describes our information theoretically optimal algorithm that has runtime exponential in the support size, $n$, and establishes Theorem~\ref{thm:main}.  Section~\ref{sec:const-k} describes our more efficient algorithm, with runtime $(nk)^{O(k)}$, which directly approximates the $k$-tensor of observations via a rank-1 tensor, establishing Theorem~\ref{thm:efficient}.   We conclude this section with a brief summary of the notation that will be used throughout the remainder of the paper.  

\subsection{Notation}

		Throughout, we use $p$ to denote a discrete distribution of support size at most $n$.
		Without loss of generality, we assume that
		the support of $p$ is $[n] = \{1, 2, \ldots, n\}$.
		Thus, $p$ can also be interpreted as a probability vector
			$(p_1, p_2, \ldots, p_n) \in \real^{n}$,
		where $p_i$ denotes the probability of element $i$.
		For a set $S \subseteq [n]$, we adopt the shorthand notation
			$p(S) = \sum_{i \in S}p_i$ to denote the total probability mass assigned to elements of $S$.

		Let $\tp{p}{k}$ denote the $k$-fold product distribution corresponding to $p$, which is the $k$th tensor power of vector $p$.
		Each entry of $\tp{p}{k}$ is indexed by $k$ indices $i_1, i_2, \ldots, i_k\in[n]$,
		and
			\[\tp{p}{k}_{i_1, i_2, \ldots, i_k} = p_{i_1} p_{i_2} \cdots p_{i_k}.\]

		More generally, a $n^k$-tensor is a $k$-dimensional array
		in which each entry is indexed by $k$ elements in $[n]$.
		The \emph{marginal} of $n^k$-tensor $A$ is the vector
		$a \in \real^{n}$ defined as
			\[a_i = \sum_{i_2, \ldots, i_k \in [n]}A_{i, i_2, \ldots, i_k}.\]
		The $i$-th \emph{slice} of $A$ is the $n^{k-1}$-tensor obtained by restricting the
		first index of $A$ to $i$.

		A probability vector (resp.~probability tensor) is a vector (resp.~tensor)
		whose entries are nonnegative and sum to $1$.
		Note that a probability $n^k$-vector defines a probability distribution on $[n]^k$.
		Moreover, its marginal is the marginal distribution of the first component,
		and its $i$-th slice (after normalization)
		is the conditional distribution of the other $k - 1$ components
		given that the first component equals $i$.

		The learning algorithm is given $m$ batches of data, each of which is a $k$-tuple in $[n]^k$. Among the $m$ batches, $m(1 - \epsilon)$ are ``good'' in the sense that each of them is drawn from $\tp{\tilde{p}}{k}$ for some distribution $\tilde{p}$ with $\Delta(p, \tilde{p}) \le \err$.\footnote{The distribution $\tilde{p}$ may vary for different good batches.} The other $m\epsilon$ batches are chosen arbitrarily after the $m(1 - \epsilon)$ good batches are drawn. The objective is to output a distribution $q$ such that $\Delta(p, q)$ is small. Here $\Delta(p, q)$ stands for the total variation distance between $p$ and $q$, i.e., one half of the $\ell_1$-distance between vectors $p$~and~$q$:
		\[
			\Delta(p, q) \coloneqq \frac{1}{2}\norm{p - q}
		=	\frac{1}{2}\sum_{i=1}^{n}|p_i - q_i| = \max_{S \subseteq [n]}\left[p(S) - q(S)\right].
		\]

%% file: bound.tex
\section{Information Theoretic Lower Bound}\label{sec:bound}
	In this section we establish Theorem~\ref{thm:lb}, showing that it is impossible to learn $p$ to an $o(\err + \epsilon/\sqrt{k})$ precision in $\ell_1$ distance, even for distributions supported on $2$ domain elements. This lower bound holds even if the learning algorithm is given unlimited computation power and access to infinitely many batches, i.e., the input of the algorithm is simply the mixture distribution $(1 - \epsilon)P + \epsilon N$, where $P$ is a mixture of $k$-fold product distributions of distributions that are $\err$-close to $p$ in the total variation distance, and $N$ is a probability $n^k$-tensor that corresponds to the distribution of the adversarial batches.

	\begin{lemma}\label{lem:bound}
		For integer $k \ge 1$ and parameters $\epsilon\in(0, 1/2)$ and $\err\in[0, 1/4)$, there are Bernoulli distributions $p$, $q$, $p'$, $q'$, together with two probability $2^k$-tensors $N^{(p)}$ and $N^{(q)}$,
		such that:
		\begin{enumerate}
			\item $\Delta(p, q) = 2\err + \epsilon / \sqrt{2k}$.
            \item $\Delta(p, p') = \Delta(q, q') = \err$.
			\item $(1 - \epsilon)\tp{p'}{k} + \epsilon N^{(p)} = (1 - \epsilon)\tp{q'}{k} + \epsilon N^{(q)}$.
		\end{enumerate}
	\end{lemma}

	Lemma~\ref{lem:bound} implies that given the distribution \[(1 - \epsilon)\tp{p'}{k} + \epsilon N^{(p)} = (1 - \epsilon)\tp{q'}{k} + \epsilon N^{(q)}\] as input, the algorithm cannot determine whether the underlying distribution is $p$ or $q$. This establishes the $\Omega(\err + \epsilon / \sqrt{k})$ lower bound in Theorem~\ref{thm:lb}.

	\begin{proof}[Proof of Lemma~\ref{lem:bound}]
		Since $\epsilon/\sqrt{2k} < \epsilon < 1/2$, by Lemma~\ref{lem:tv-upper}, for Bernoulli distributions $p'$ and $q'$ with means $(1-\epsilon/\sqrt{2k})/2$ and $(1+\epsilon/\sqrt{2k})/2$, it holds that $\Delta(p', q') = \epsilon / \sqrt{2k}$ and $\Delta(\tp{p'}{k}, \tp{q'}{k}) \le \epsilon$. Let $p$ and $q$ be Bernoulli distributions with means $(1-\epsilon/\sqrt{2k})/2 - \err$ and $(1+\epsilon/\sqrt{2k})/2 + \err$.\footnote{$p$ and $q$ are well defined since $\epsilon/(2\sqrt{2k}) + \err < 1/4 + 1/4 = 1/2$.} Then the distributions clearly satisfy the first two conditions.
        
        Let $A$ be the entrywise maximum of $\tp{p'}{k}$ and $\tp{q'}{k}$, i.e., $A_i = \max(\tp{p'}{k}_i, \tp{q'}{k}_i)$ for every $i \in [2] ^ k$. Let $\alpha$ denote the sum of entries in $A$. Then, \[\alpha = 1 + \Delta(\tp{p'}{k}, \tp{q'}{k}) \le 1 + \epsilon \le 1 / (1 - \epsilon).\] Define $N^{(p)} = \left[A/\alpha - (1 - \epsilon)\tp{p'}{k}\right]/\epsilon$ and $N^{(q)} = \left[A/\alpha - (1 - \epsilon)\tp{q'}{k}\right]/\epsilon$. Then the third condition is met, and it remains to prove that $N^{(p)}$ and $N^{(q)}$ are probability tensors.
        
        Note that the elements in $N^{(p)}$ sum to $(\alpha / \alpha - (1 - \epsilon)\cdot 1)/\epsilon = 1$. Moreover, since $\alpha \le 1 / (1 - \epsilon)$,
		\[
			N^{(p)}_{i}
		=	A_i / \alpha - (1 - \epsilon)\tp{p'}{k}_i
		\ge (1 - \epsilon) (A_i - \tp{p'}{k}_i) \ge 0
		\]
		for any $i \in [2] ^ k$. This shows that $N^{(p)}$ and $N^{(q)}$ are probability tensors and finishes the proof.
	\end{proof}

%% file: const_n.tex
\section{An Information Theoretically Optimal Algorithm}\label{sec:const-n}
	In this section, we present an algorithm that approximates the distribution to an information theoretically optimal $O(\err + \epsilon/\sqrt{k})$ accuracy using $O((n + k + \ln(1/\delta)) / \epsilon^2)$ batches. The algorithm runs in time $\poly(2^n, k, 1/\epsilon, 1/\err, \ln(1/\delta))$. In particular, the algorithm is computationally efficient if the distribution has a relatively small support.

	Our approach is to reduce the problem to learning Bernoulli distributions: we estimate the probability mass of any subset of the support to an $O(\err + \epsilon/\sqrt{k})$ accuracy, and output a distribution that is consistent with the measurements. Then the total variation distance between our output and the true distribution would also be upper bounded by $O(\err + \epsilon/\sqrt{k})$.

    In the following we define a subroutine that efficiently estimates $p(S)$, the probability mass that $p$ assigns to set $S \subseteq [n]$. Given batches $x_1, x_2, \ldots, x_m \in [n] ^ k$, the algorithm counts the number of batches that contain exactly $i$ elements in $S$ ($0 \le i \le k$) and obtains a distribution $f^{S}$ over $\{0, 1, \ldots, k\}$. Then it outputs $(i + 2)\err$ for some $0 \le i \le 1/\err - 4$ as the estimation, if there is a mixture of binomial distributions with success probabilities in interval $[i\err, (i + 4)\err]$ such that its $\ell_1$ distance to $f^{S}$ is upper bounded by $O(\epsilon)$.

\begin{algorithm}[H]\label{alg:bi-est}
\KwIn{Batches $x_1, x_2, \ldots, x_m \in [n]^k$, set $S \subseteq [n]$, and parameters $\epsilon, \err$.}
\KwOut{An estimation of $p(S)$.}
\caption{$\biest((x_i)_{i\in[m]}, S, \epsilon, \err)$}
\For {$i = 1, 2, \ldots, m$} {
	$\cnt_i \gets \sum_{j=1}^{k}\ind{x_{i,j}\in S}$\;
}
\For {$i = 0, 1, \ldots, k$} {
$f^{S}_i \gets \frac{1}{m}\sum_{j=1}^{m}\ind{\cnt_j = i}$\;
}
$\tot \gets 4\err/(\epsilon/k)$\;
\For {$i = 0, 1, \ldots, 1/\err - 4$} {
\Return $(i + 2)\err$ if the following mathematical program is feasible:\\
\begin{equation}\begin{split}\label{eq:biest-lp}
\textrm{find}\quad&\alpha_0, \alpha_1, \ldots, \alpha_{\tot}\\
\textrm{subject to}\quad& \Delta\left(\sum_{j=0}^{\tot}\alpha_j B\left(k, i\err + j\epsilon/k\right), f^{S}\right) \le 2\epsilon\\
& \sum_{j=0}^{\tot}\alpha_j = 1\\
& \alpha_j \ge 0,~\forall j \in \{0, 1, \ldots, \tot\}
\end{split}\end{equation}
}
\end{algorithm}

	Note that the mathematical program \eqref{eq:biest-lp} can be transformed to an equivalent linear program, and therefore can be efficiently solved.

	The following lemma bounds the difference between the estimation $\biest((x_i)_{i\in[m]}, S, \epsilon, \err)$ and $p(S)$.
	\begin{lemma}\label{lem:bi-est}
		Suppose that $S \subseteq [n]$,
		$\epsilon \in (0, 1 / 900)$,
		and $\err, \delta \in (0, 1)$.
		For some
			$m = O((k + \ln(1/\delta))/\epsilon ^ 2)$, $x_1, x_2, \ldots, x_m$ are $m$ batches chosen as in the setup of Theorem~\ref{thm:main}.
		With probability $1 - \delta$ over the randomness of the $m(1 - \epsilon)$ good batches,
			\[\left|\biest((x_i)_{i\in[m]}, S, \epsilon, \err) - p(S)\right| \le 3\err + 60\epsilon/\sqrt{k}.\]
		Here the ``O'' notation hides an absolute constant.
	\end{lemma}
	\begin{proof}[Proof of Lemma~\ref{lem:bi-est}]
		Let $I\subseteq[m]$ denote the indices of the $m(1 - \epsilon)$ good batches and let $\overline{I}$ be its complement.
        Define distributions $\hat{p}^{S}$ and $\delta^{S}$ over $\{0, 1, \ldots, k\}$ as:
        \[\hat{p}^{S}_i = \frac{1}{|I|}\sum_{j\in I}\ind{\cnt_j = i}\]
        and
        \[\delta^{S}_i = \frac{1}{\left|\overline{I}\right|}\sum_{j\in \overline{I}}\ind{\cnt_j = i}.\]
        Then the distribution $f^{S}$ defined in Algorithm~\ref{alg:bi-est} can be rewritten as
        \[f^{S} = (1 - \epsilon)\hat{p}^{S} + \epsilon\delta^{S}.\]
        
        \paragraph{Concentration of $\hat{p}^{S}$.}
        For each good batch $j \in I$, let $\theta_j$ denote the probability of set $S$ in the actual distribution from which batch $j$ is drawn. By the assumption that the underlying distribution of each good batch is $\err$-close to the target $p$, we have $|\theta_j - p(S)| \le \err$.
        Define $p^{S}$ as the uniform mixture of binomial distributions $B(k, \theta_j)$ for $j \in I$, i.e.,
        \[p^{S} = \frac{1}{|I|}\sum_{j\in I}B(k, \theta_j).\]
		Since $\hat{p}^{S}$ denotes the frequency of $|I| = \Omega(m)$ samples, exactly one of which is drawn from $B(k, \theta_j)$ for each $j \in I$, for some $m = O((k + \ln(1/\delta))/\epsilon ^ 2)$,
		it holds with probability $1 - \delta$ that
			$\Delta(p^{S}, \hat{p}^{S}) \le \epsilon/2$.
		It follows that
		\[
			\Delta(p^{S}, f^{S})
		\le	\Delta(p^{S}, \hat{p}^S) + \Delta(\hat{p}^S, f^{S})
		=	\Delta(p^{S}, \hat{p}^S) + \epsilon \Delta(\hat{p}^S, \delta^{S})
		\le 3\epsilon/2.
		\]

        \paragraph{Feasibility of mathematical program \eqref{eq:biest-lp}.}
        Let $i^* = \lfloor p(S)/\err \rfloor - 2$. We show that with probability $1 - \delta$, the mathematical program \eqref{eq:biest-lp} is feasible for $i = i^*$.

        Since $|(i^* + 2)\err - p(S)| \le \err$ and $|\theta_j - p(S)| \le \err$, it holds for any $j \in I$ that
        \[i^*\err \le \theta_j \le (i^* + 4)\err.\]
        Let $\tilde\theta_j$ be the value among $\{i^*\err + t\epsilon/k: t\in\{0, 1, \ldots, \tot\}\}$ that is closest to $\theta_j$. By definition, we have $|\tilde\theta_j - \theta_j| \le \epsilon/(2k)$, which implies that
        	\[\Delta(B(k, \theta_j), B(k, \tilde\theta_j)) \le k\cdot\epsilon/(2k) = \epsilon/2.\]
       	Let $\tilde{p}^{S} = \frac{1}{|I|}\sum_{j\in I}B(k, \tilde\theta_j).$
        Then we have
        	\[\Delta(p^{S}, \tilde{p}^{S})\le\frac{1}{|I|}\sum_{j\in I}\Delta(B(k, \theta_j), B(k, \tilde\theta_j))\le\epsilon/2,\]
        and it follows that with probability $1-\delta$,
        	\[\Delta(\tilde{p}^{S}, f^{S})\le\Delta(p^{S}, \tilde{p}^{S})+\Delta(p^{S}, f^{S})\le \epsilon/2 + 3\epsilon/2 = 2\epsilon.\]
        Note that this naturally defines a feasible solution to the mathematical program \eqref{eq:biest-lp} for $i = i^*$: for each $t \in \{0, 1, \ldots, \tot\}$,
        \[\alpha_t = \frac{1}{|I|}\sum_{j\in I}\ind{\tilde\theta_j = i^*\err + t\epsilon/k}.\]

		\paragraph{Approximation guarantee.}
        Let $x_0$ denote $\biest((x_i)_{i\in[m]}, S, \epsilon, \err)$. In the following we prove that, with probability $1 - \delta$, $x_0$ is a good approximation of $p(S)$.

        Recall that $p^{S} = \frac{1}{|I|}\sum_{j\in I}B(k, \theta_j)$ is a mixture of binomial distributions with success probabilities in $[p(S) - \err, p(S) + \err]$. Moreover, by definition of procedure $\biest$, $f^{S}$ is $2\epsilon$-close to a mixture of binomial distributions with success probabilities that lie in $[x_0 - 2\err, x_0 + 2\err]$. Then the inequality $\Delta(p^S, f^S) \le 3\epsilon/2$, which holds with probability $1 - \delta$, further implies that the two mixtures above are $4\epsilon$-close to each other.

		Let $\epsilon' = 60\epsilon/\sqrt{k}$. Note that our assumption $\epsilon < 1/900$ implies that $\epsilon' < 1/(15\sqrt{k})$. Now suppose for a contradiction that $|x_0 - p(S)| > 3\err + \epsilon'$. Without loss of generality, we have $x_0 > p(S) + 3\err + \epsilon'$, or equivalently, $x_0 - 2\err > p(S) + \err + \epsilon'$. Applying the contrapositive of Lemma~\ref{lem:tv-lower-01} with parameters $\epsilon = \epsilon'$, $p = p(S) + \err$ and $q = x_0 - 2\err$ gives a contradiction.
	\end{proof}

	Now we prove our main theorem by Lemma~\ref{lem:bi-est} and a reduction to the estimation of Bernoulli random variables.
	\begin{proof}[Proof of Theorem~\ref{thm:main}]
		We compute $\hat{p}(S) = \biest((x_i)_{i\in[m]}, S)$ for each $S \subseteq [n]$,
		and then output an arbitrary feasible solution (if any) to the following linear program:
		\begin{equation}\begin{split}\label{eq:const-n-lp}
			\textrm{find}\quad&q\in\real^{n}\\
			\textrm{subject to}\quad& \left|\sum_{i\in S}q_i - \hat{p}(S)\right| \le 3\err + 60\epsilon/\sqrt{k},~\forall S \subseteq [n]\\
			&\sum_{i=1}^{n}q_i = 1\\
			&q_i \ge 0,~\forall i \in [n]\\
		\end{split}\end{equation}
		The algorithm described above involves
		solving a linear program with $n$ variables and $O(2 ^ n)$ constraints,
		as well as $2^n$ calls to the subroutine $\biest$,
		each of which takes polynomial time.
		Thus the whole algorithm runs in $\poly(2 ^ n, k, 1/\epsilon, 1/\err, \log(1/\delta))$ time.

		Let $\delta_0 = \delta / {2 ^ n}$. By Lemma~\ref{lem:bi-est} and a union bound, for some \[m = O\left((k + \ln(1/\delta_0))/\epsilon^2\right) = O\left((n + k + \ln(1/\delta)) / \epsilon^2\right),\] it holds with probability $1 - 2^n\cdot\delta_0 = 1 - \delta$ that, for any $S \subseteq [n]$, \[\left|\hat{p}(S) - p(S)\right| \le 3\err + 60\epsilon/\sqrt{k}.\] This implies that linear program \eqref{eq:const-n-lp} is feasible with probability $1 - \delta$. Moreover, let $q$ be any feasible solution to \eqref{eq:const-n-lp}. Then it holds for any $S\subseteq[n]$ that \[\left|p(S) - q(S)\right| \le \left|p(S) - \hat{p}(S)\right| + \left|q(S) - \hat{p}(S)\right| \le 6\err + 120\epsilon/\sqrt{k}.\] And hence, from the definition of total variation distance, $\Delta(p, q) = O(\err + \epsilon / \sqrt{k}).$
	\end{proof}

%% file: const_k.tex
\section{A Tensor-Based Algorithm}\label{sec:const-k}
	In this section, we present a different algorithm that efficiently learns the distribution in the small-batch regime where $k$ is a constant, and where all the good batches are drawn from the actual distribution $p$, i.e., $\err = 0$. The algorithm works on the \emph{frequency tensor} $A$ defined by the $m$ batches, i.e., $A_{i_1, i_2, \ldots, i_k}$ is the fraction of batches whose set of samples equal $(i_1, i_2, \ldots, i_k)$.

	The following recursive function $\distset(n, k, A)$ takes tensor $A$ and outputs a set of ``guesses'' of the distribution.

	\begin{algorithm}[H]
		\KwIn{$n$, $k$, and $n^k$-tensor $A$.}
		\KwOut{A set of $n$-dimensional vectors.}
		\caption{$\distset(n, k, A)$}
		\If{$k = 1$} {
			\Return $\{A\}$\;
		}
		\For {$i = 1, 2, \ldots, n$} {
			$A_i \gets $ the normalized $i$-th slice of $A$\;
			$S_i \gets \distset(n, k - 1, A_i)$\;
		}
		$a \gets$ the marginal of $A$\;
		\Return $S_1 \cup S_2 \cup \cdots \cup S_n \cup \{a\}$\;
	\end{algorithm}

	The following lemma states that given sufficiently many batches,
	$\distset(n, k, A)$ contains a vector that is $O(\epsilon / k)$-close
	to the real distribution $p$ in the total variation distance.
	\begin{lemma}\label{lem:distset-finite}
		Suppose that $\epsilon\in(0, 1 / 2)$ and
		vector $p$ is a probability distribution on $[n]$.
		Let $A \in \real^{n^k}$ be the frequency tensor of $m$ batches,
		among which $m(1-\epsilon)$ are drawn from $\tp{p}{k}$
		and the other $m\epsilon$ batches are arbitrary
		and may depend on the $m(1-\epsilon)$ good batches.
		Then, with probability $1 - \delta$ (over the randomness in the good batches),
		\[
			\min_{q\in\distset(n, k, A)}\Delta(p, q)
		\le	\frac{6\epsilon}{k} + O\left(\sqrt\frac{n^k \cdot k!\cdot(n + k\ln n + \ln(1/\delta))}{m}\right).
		\]
	\end{lemma}

	\begin{proof}[Proof of Lemma~\ref{lem:distset-finite}]
		Let $\delta_0 = \delta / (n+1)^{k-1}$. We recursively define functions $f_{n, 1}(\epsilon, m)$ through $f_{n, k}(\epsilon, m)$ for $\epsilon, m > 0$ as follows:
		\begin{equation}\label{eq:recur-1}
			f_{n, 1}(\epsilon, m) = 3\epsilon + C\sqrt{\frac{n + \ln(1/\delta_0)}{m}}
		\end{equation}
		and for $t \ge 2$,
		\begin{equation}\label{eq:recur-t}
			f_{n, t}(\epsilon, m) = \max\left(
				\frac{3\epsilon}{t} + C\sqrt{\frac{n + \ln(1/\delta_0)}{m}},
				f_{n, t - 1}\left(1 - (1 - \epsilon)^{(t-1)/t}, \frac{m}{nt}\right)
			\right).
		\end{equation}
		Here $C$ is an absolute constant to be determined later.

		The following claim states that $\min_{q\in\distset(n, k, A)}\Delta(p, q)$
		is upper bounded by $f_{n, k}(\epsilon, m(1 - \epsilon))$ with high probability.
		\begin{claim}\label{claim:fnt}
			Suppose that $t \in [k]$, $\epsilon \in (0, 1 / 2)$,
			and $A\in\real^{n^t}$ is the frequency tensor of a data set,
			in which at least $m$ batches are drawn from $\tp{p}{t}$,
			and the fraction of these batches is at least $1 - \epsilon$.
			Then, with probability $1 - (n + 1)^{t - 1}\cdot\delta_0$,
			it holds that
			\[
				\min_{q\in\distset(n, t, A)}\Delta(p, q) \le f_{n, t}(\epsilon, m).
			\]
		\end{claim}
		The following claim further upper bounds $f_{n, k}(\epsilon, m)$.
		\begin{claim}\label{claim:fnt-bound}
		For any $\epsilon \in (0, 1 / 2)$ and $m > 0$,
		\[
			f_{n, k}(\epsilon, m)
		\le \frac{6\epsilon}{k} + C\sqrt\frac{n^k\cdot k!\cdot(n + \ln(1/\delta_0))}{m}.
		\]
		\end{claim}
		Claims \ref{claim:fnt}~and~\ref{claim:fnt-bound} together imply that
		with probability
			\[1 - (n+1)^{k-1} \cdot \delta_0 = 1 - \delta,\]
		it holds that
		\[
			\min_{q\in\distset(n, k, A)}\Delta(p, q)
		\le	f_{n, k}(\epsilon, m(1 - \epsilon))
		\le	\frac{6\epsilon}{k} + O\left(\sqrt\frac{n^k\cdot k!\cdot(n + k\ln n + \ln(1/\delta))}{m}\right).
		\]
	\end{proof}

	Lemma~\ref{lem:distset-finite} implies that given
		\[m = O(n^k\cdot k!\cdot(n + k\ln n + \ln(1/\delta))\cdot k^2/\epsilon^2) = (nk)^{O(k)}\ln(1/\delta)/\epsilon^2\]
	batches, the minimum distance is upper bounded by $O(\epsilon / k)$.
	Now we prove Theorem~\ref{thm:efficient} by leveraging this approximation guarantee of procedure $\distset$.

	\begin{proof}[Proof of Theorem~\ref{thm:efficient}]
		Let
			\[q_0 = \argmin_{q \in \distset(n, k, A)}\Delta(A, \tp{q}{k}),\]
		and
			\[q_1 = \argmin_{q \in \distset(n, k, A)}\Delta(p, q).\]
		Note that $q_0$ can be computed from $A$ in $(nk)^{O(k)}/\epsilon^2$ time.
		In the following we prove that $q_0$ is a good approximation of $p$.

		By Lemma~\ref{lem:distset-finite},
		for some $m = (nk)^{O(k)}\ln(1/\delta)/\epsilon^2$,
		it holds with probability $1 - \delta/2$ that
			$\Delta(p, q_1) \le 7\epsilon/k$,
		and thus,
			$\Delta(\tp{p}{k}, \tp{q_1}{k}) \le k\Delta(p, q_1) \le 7\epsilon.$

		We write $A$ as
			$A = (1 - \epsilon)\hat{P} + \epsilon N$,
		where $\hat{P}$ and $N$ are the frequency tensor
		of the ``good'' and ``bad'' batches, respectively.
		Since $\hat{P}$ denotes the frequency among $m(1-\epsilon) = \Omega(m)$ samples
		drawn from $\tp{p}{k}$, and the support of $\tp{p}{k}$ is of size $n^k$,
		for some $m = O\left((n^k+\ln(1/\delta))/\epsilon^2\right)$,
		it holds with probability $1-\delta/2$ that
			$\Delta(\hat{P}, \tp{p}{k}) \le \epsilon$.
		Therefore,
		\[
			\Delta(A, \tp{p}{k})
		\le	\Delta(A, \hat{P}) + \Delta(\hat{P}, \tp{p}{k})
		=	\epsilon\Delta(\hat{P}, N) + \Delta(\hat{P}, \tp{p}{k})
		\le 2\epsilon.
		\]
		By a union bound, it holds with probability $1 - \delta$ that
		\[
			\Delta(A, \tp{q_1}{k})
		\le	\Delta(A, \tp{p}{k}) + \Delta(\tp{p}{k}, \tp{q_1}{k})
		\le 2\epsilon + 7\epsilon
		=	9\epsilon.
		\]
		Furthermore, by definition of $q_0$,
			$\Delta(A, \tp{q_0}{k}) \le \Delta(A, \tp{q_1}{k}) \le 9\epsilon$,
		and thus
		\[
			\Delta(\tp{p}{k}, \tp{q_0}{k})
		\le	\Delta(A, \tp{p}{k}) + \Delta(A, \tp{q_0}{k})
		\le	2\epsilon + 9\epsilon = 11\epsilon.
		\]

		Let $\epsilon' = 165\epsilon/\sqrt{k} < 1/(15\sqrt{k})$.
		Since
			$\Delta(\tp{p}{k}, \tp{q_0}{k}) \le 11\epsilon = \epsilon'\sqrt{k}/15$,
		applying the contrapositive of Lemma~\ref{lem:tv-lower} with parameter $\epsilon'$
		yields that, with probability at least $1 - \delta$,
			\[\Delta(p, q_0) \le \epsilon' = O(\epsilon/\sqrt{k}).\]
	\end{proof}

%% file: tensor.tex
\section{Technical Lemmas}
	We prove a few useful technical lemmas.
	\begin{lemma}\label{lem:1-eps-add}
		For any $\epsilon\in[0,1/2]$ and $\alpha\in[0,1]$,
			$(1 - \epsilon)^{\alpha} \ge 1 - 2\alpha\epsilon$.
	\end{lemma}
	\begin{proof}[Proof of Lemma~\ref{lem:1-eps-add}]
		Let $f(x) = (1 - x)^{\alpha} - (1 - 2\alpha x)$.
		Since $f'(x) = \alpha[2 - (1 - x)^{\alpha - 1}] \ge 0$
		for any $x \in [0, 1/2]$ and $\alpha \in [0, 1]$,
		it holds for any $\epsilon \in [0, 1/2]$ that $f(\epsilon) \ge f(0) = 0$,
		which proves the lemma.
	\end{proof}
	\begin{lemma}\label{lem:(1-a/x)^x}
		For any $a > 0$, the function $(1 - a/x)^x$ is increasing on $(a, +\infty)$.
	\end{lemma}
	\begin{proof}[Proof of Lemma~\ref{lem:(1-a/x)^x}]
		Let $f(x) = x\ln(1 - a/x)$. Then for $x \in (a, +\infty)$,
			\[f'(x) = \ln\left(1 - \frac{a}{x}\right) + \frac{a}{x - a}\]
		and
			\[f''(x) = -\frac{a^2}{x(x-a)^2} < 0.\]
		Since
			$\lim_{x\to+\infty}f'(x) = 0$,
		$f'$ is positive on $(a, +\infty)$, which proves the lemma.
	\end{proof}
	\begin{lemma}\label{lem:(1+eps)(1-eps)}
		Suppose that $n$ and $m$ are positive integers such that $m \ge \max(n, 2)$.
		Then for any $\alpha \in [0, 1.1\sqrt{n}]$,
		\[
			\left(1 + \frac{\alpha}{n}\right)^n\left(1 - \frac{\alpha}{m}\right)^m
		\ge	\frac{1}{7}.
		\]
	\end{lemma}
	\begin{proof}[Proof of Lemma~\ref{lem:(1+eps)(1-eps)}]
		For $x \in [0, 1.1\sqrt{n}]$, define
			\[f_{n, m}(x) = n\ln\left(1 + \frac{x}{n}\right) + m\ln\left(1 - \frac{x}{m}\right).\]
		Note that $f_{n, m}$ is well-defined since
			$x \le 1.1\sqrt{n} \le 1.1\sqrt{m} < m$.
		Moreover,
			\[f'_{n, m}(x) = 1 / (1 + x / n) - 1 / (1 - x / m) \le 0.\]
		Thus, it remains to prove the inequality for $\alpha = 1.1\sqrt{n}$,
		i.e.,
		\begin{equation}\label{eq:n-m}
			\left(1 + \frac{1.1}{\sqrt{n}}\right)^n\left(1 - \frac{1.1\sqrt{n}}{m}\right)^m
		\ge	\frac{1}{7}.
		\end{equation}

		If $n \ge 2$, we lower bound the lefthand side of \eqref{eq:n-m} by
		\[
			\left(1 + \frac{1.1}{\sqrt{n}}\right)^n\left(1 - \frac{1.1\sqrt{n}}{m}\right)^m
		\ge	\left(1 + \frac{1.1}{\sqrt{n}}\right)^n\left(1 - \frac{1.1\sqrt{n}}{n}\right)^n
		=	\left(1 - \frac{1.21}{n}\right)^n
		\ge	\left(1 - \frac{1.21}{2}\right)^{2}
		\ge	\frac{1}{7}.
		\]
		Here the first and third steps follow from Lemma~\ref{lem:(1-a/x)^x}.
		For $n = 1$, we have
		\[
			\left(1 + \frac{1.1}{\sqrt{n}}\right)^n\left(1 - \frac{1.1\sqrt{n}}{m}\right)^m
		=	2.1\times\left(1 - \frac{1.1}{m}\right)^m
		\ge	2.1\times\left(1 - \frac{1.1}{2}\right)^2
		\ge	\frac{1}{7}.
		\]
	\end{proof}
	\begin{lemma}\label{lem:binom}
	For any $k\in\mathbb{N}$ and $t\in\{1, \ldots, k - 1\}$,
	\[
		\binom{k}{t}\left(\frac{t}{k}\right)^t\left(\frac{k-t}{k}\right)^{k-t}
	\ge \frac{1}{3\sqrt{t}}.
	\]
	\end{lemma}
	\begin{proof}[Proof of Lemma~\ref{lem:binom}]
		Stirling's approximation states that
		\[\sqrt{2\pi n}\cdot(n/e)^n\le n!\le e\sqrt{n}\cdot(n/e)^n\]
		for any positive integer $n$.
		Thus,
		\[
			\binom{k}{t} = \frac{k!}{t!(k-t)!}
		\ge \frac{\sqrt{2\pi k}}{e^2\sqrt{t(k-t)}}\cdot\frac{(k/e)^k}{(t/e)^t[(k-t)/e]^{k-t}}
		=	\frac{\sqrt{2\pi}}{e^2}\cdot\sqrt\frac{k}{t(k-t)}\cdot\frac{k^k}{t^t(k-t)^{k-t}}.
		\]
		We conclude that
		\[
			\binom{k}{t}\left(\frac{t}{k}\right)^t\left(\frac{k-t}{k}\right)^{k-t}
		\ge \frac{\sqrt{2\pi}}{e^2}\cdot\sqrt\frac{k}{t(k-t)}
		\ge \frac{1}{3\sqrt{t}}.
		\]
	\end{proof}

\section{Tensorization of the Total Variation Distance}\label{sec:tensor}
	In this section we prove two inequalities regarding the total variation distance between $k$-fold product distributions.

	\begin{lemma}\label{lem:tv-upper}
		For any $\epsilon\in(0, 1/2)$ and $k\in\mathbb{N}$,
		there exist Bernoulli distributions $P$ and $Q$ such that:
		\begin{enumerate}
			\item $\Delta(P, Q) = \epsilon$.
			\item $\Delta(\tp{P}{k}, \tp{Q}{k}) \le \epsilon\sqrt{2k}$.
		\end{enumerate}
		In particular, Bernoulli distributions with means $(1-\epsilon)/2$ and $(1+\epsilon)/2$ satisfy the above conditions.
	\end{lemma}
	\begin{proof}[Proof of Lemma~\ref{lem:tv-upper}]
		Let $P$ and $Q$ be Bernoulli distributions with means $(1-\epsilon)/2$ and $(1+\epsilon)/2$,
		respectively. Clearly, we have $\Delta(P, Q) = \epsilon$ and
		\[
			\KL(P, Q)
		=	\epsilon\ln\frac{1 + \epsilon}{1 - \epsilon}
		\le \frac{2\epsilon^2}{1 - \epsilon}
		\le 4\epsilon^2.
		\]
		Here the second step applies the inequality $\ln(1 + x) \le x$.
		By Pinsker's inequality,
		\[
			\Delta(\tp{P}{k}, \tp{Q}{k})
		\le	\sqrt{\frac{1}{2}\KL(\tp{P}{k}, \tp{Q}{k})}
		=	\sqrt{\frac{k}{2}\KL(P, Q)}
		\le	\epsilon\sqrt{2k}.
		\]
	\end{proof}

	The next lemma shows that for sufficiently small $\epsilon$,
	the $O(\sqrt{k})$ ratio between $\Delta(\tp{P}{k}, \tp{Q}{k})$ and $\Delta(P, Q)$ in Lemma~\ref{lem:tv-upper}
	is tight (up to a constant factor).

	\begin{lemma}\label{lem:tv-lower}
		Suppose that $k\in\mathbb{N}$ and $\epsilon\in\left(0, 1/(15\sqrt{k})\right)$.
		$P$ and $Q$ are two distributions on the same support such that
			$\Delta(P, Q) \ge \epsilon$.
		Then,
			\[\Delta(\tp{P}{k}, \tp{Q}{k}) \ge \frac{\epsilon\sqrt{k}}{15}.\]
	\end{lemma}

	To prove Lemma~\ref{lem:tv-lower}, we have the following weaker claim
	for the special case of Bernoulli distributions.

	\begin{lemma}\label{lem:tv-lower-01}
		Suppose that $k\in\mathbb{N}$, $\epsilon \in (0, 1/(15\sqrt{k}))$,
		$p, q \in [0, 1]$ and $q - p \ge \epsilon$.
        $\tilde{P}$ is a mixture of binomial distributions with $k$ trials and success probabilities in $[0, p]$, i.e.,
        $\tilde{P} = \sum_{i=1}^{\tot}\alpha_iB(k, p_i)$ for some nonnegative weights $\alpha_1, \ldots, \alpha_{\tot}$ and probabilities $p_1, \ldots, p_{\tot}$ in $[0, p]$, such that $\sum_{i=1}^{\tot}\alpha_i=1$.
        $\tilde{Q}$ is a mixture of binomial distributions with $k$ trials and success probabilities in $[q, 1]$.
		Then,
			\[\Delta(\tilde{P}, \tilde{Q}) \ge \frac{\epsilon\sqrt{k}}{15}.\]
	\end{lemma}

	We prove Lemma~\ref{lem:tv-lower} by a reduction to Bernoulli distributions.
	\begin{proof}[Proof of Lemma~\ref{lem:tv-lower}]
		Suppose that $P$ and $Q$ share the support $[n]$.
		Define $\pi:[n]\to\{0, 1\}$ as
		\[
			\pi(x) = \begin{cases}
				1, & P(x) \ge Q(x),\\
				0, & P(x) < Q(x).
			\end{cases}
		\]
		Let $p$ and $q$ the means of $\pi(P)$ and $\pi(Q)$. Without loss of generality, $p \le q$.
		By construction, we have
			$|p - q| = \Delta(P, Q) \ge \epsilon$,
		and
		\begin{align*}
			\Delta(B(k, p), B(k, q))
		=	&\Delta((\pi(P_1), \pi(P_2), \ldots, \pi(P_k)), (\pi(Q_1), \pi(Q_2), \ldots, \pi(Q_k)))\\
		\le	&\Delta((P_1, P_2, \ldots, P_k), (Q_1, Q_2, \ldots, Q_k))\\
		=	&\Delta(\tp{P}{k}, \tp{Q}{k}),
		\end{align*}
		where $(P_i)_{i\in[k]}$ and $(Q_i)_{i\in[k]}$ are independent copies of $P$ and $Q$. Here the second step applies the data processing inequality.
		Applying Lemma~\ref{lem:tv-lower-01} with $\tilde{P} = B(k, p)$ and $\tilde{Q} = B(k, q)$ yields that
		\[
			\Delta(\tp{P}{k}, \tp{Q}{k})
		\ge	\Delta(B(k, p), B(k, q))
		\ge	\epsilon\sqrt{k} / 15.
		\]
	\end{proof}

	Now we turn to the more technical proof of Lemma~\ref{lem:tv-lower-01}.
	\begin{proof}[Proof of Lemma~\ref{lem:tv-lower-01}]
		If $k < 10$,
		the inequality trivially follows from
		\[
			\Delta(\tilde{P}, \tilde{Q})
		\ge	|p - q|
		> \epsilon\sqrt{k}/15.
		\]
		Thus we assume that $k \ge 10$ in the following proof.
		Without loss of generality, we have $p \le 1 / 2$ (otherwise we prove the lemma for $p' = 1 - q \le 1/2$ and $q' = 1 - p$).

		Let $t = \lfloor p(k - 1) \rfloor$.
		Note that $t \le (k - 1) / 2$ and $t \le p(k - 1) < t + 1$.
		Define function
			\[f(x) = \sum_{j = 0}^{t}\binom{k}{j}x^j(1-x)^{k-j}.\]
		Note that
		\begin{equation}\begin{split}\label{eq:f-prime}
			f'(x)
		=	&k\sum_{j=0}^{t}\left[\binom{k-1}{j-1}x^{j-1}(1-x)^{k-j}-\binom{k-1}{j}x^j(1-x)^{k-j-1}\right]\\
		=	&-k\binom{k-1}{t}x^t(1-x)^{k-1-t} \le 0,
		\end{split}\end{equation}
		and thus $f$ is non-increasing.

		Since $f(x)$ is the probability that the binomial distribution $B(k, x)$ assigns to set $\{0, 1, \ldots, t\}$, $\tilde{P}(\{0, 1, \ldots, t\})$ can be written as a weighted average of $f(p_i)$'s for $p_1, p_2, \ldots\in[0, p]$. Similarly, $\tilde{Q}(\{0, 1, \ldots, t\})$ is a weighted average of $f(q_i)$'s for $q_1, q_2, \ldots \in [q, 1]$. Since 
        \[\tilde{P}(\{0, 1, \ldots, t\}) - \tilde{Q}(\{0, 1, \ldots, t\})\]
        is a lower bound on $\Delta(\tilde{P}, \tilde{Q})$, it remains to show that $f(p_i)-f(q_i)\ge \epsilon\sqrt{k} / 15$ for any $p_i \le p$ and $q_i \ge q$. The monotonicity of $f$ and the fact that $q \ge p + \epsilon$ further imply that it suffices to prove
        \[f(p) - f(p + \epsilon) \ge \epsilon\sqrt{k} / 15.\]
        We prove the inequality in the following two cases.

		\paragraph{Case 1: $t = 0$.}
		In this case, we have $0 \le p < 1 / (t - 1)$.
		Note that
			\[f(p) - f(p + \epsilon) = -\epsilon f'(x) = \epsilon k(1 - x)^{k - 1}\]
		for some $x \in (p, p + \epsilon)$.
		If $\epsilon \le 1 / k$,
		we have $x \le p + \epsilon \le 2 / (k - 1)$,
		and
		\begin{equation}\label{eq:t=0}
			f(p) - f(p + \epsilon)
		\ge	\epsilon k\left(1 - \frac{2}{k - 1}\right)^{k - 1}
		\ge \epsilon k\left(1 - \frac{2}{9}\right)^{9}
		\ge \frac{\epsilon k}{10}.
		\end{equation}
		Here the second step follows from Lemma~\ref{lem:(1-a/x)^x}
		and the assumption that $k \ge 10$.
		It then follows that $f(p) - f(p + \epsilon) \ge \epsilon\sqrt{k} / 15$.

		If $\epsilon \ge 1 / k$, by Inequality~\eqref{eq:t=0}
		and the assumption that $\epsilon < 1 / (15\sqrt{k})$,
		\[
			f(p) - f(p + \epsilon)
		\ge	f(p) - f(p + 1 / k)
		\ge \frac{1}{10}
		\ge \frac{\epsilon\sqrt{k}}{15}.
		\]

		\paragraph{Case 2: $t > 0$.}
		Let $x_0 = t/(k-1)$. By Equation~\eqref{eq:f-prime} and Lemma~\ref{lem:binom}, we have
		\[
			|f'(x_0)|
		=	k\binom{k-1}{t}\left(\frac{t}{k-1}\right)^t\left(\frac{k-1-t}{k-1}\right)^{k-1-t}
		\ge \frac{k}{3\sqrt{t}}.
		\]

		For any $x\in[p, p + \epsilon]$, we can write
			\[x = (t + \alpha) / (k - 1),\]
		for some $\alpha \ge 0$.
		Then,
		\[
			\frac{|f'(x)|}{|f'(x_0)|}
		=	\left(1 + \frac{\alpha}{t}\right)^t\left(1 - \frac{\alpha}{k - 1 - t}\right)^{k - 1 - t}
		\]

		Since $t \le (k - 1) / 2$ and $k \ge 10$, 
		we have $k - 1 - t \ge \max(t, (k - 1) / 2) \ge \max(t, 2)$.
		Applying Lemma~\ref{lem:(1+eps)(1-eps)}
		with $n = t$ and $m = k - 1 - t$ yields that
		\[
			\left(1 + \frac{\alpha}{t}\right)^t\left(1 - \frac{\alpha}{k - 1 - t}\right)^{k - 1 - t}
		\ge	\frac{1}{7}
		\]
		for any $\alpha \in \left[0, 1.1\sqrt{t}\right]$.
		Consequently,
		\[
			|f'(x)| \ge \frac{|f'(x_0)|}{7} \ge \frac{k}{21\sqrt{t}}
		\]
		for any $x \in [x_0, x_0 + 1.1\sqrt{t} / (k - 1)]$.

		Let $l$ be the length of the intersection of
			$[p, p + \epsilon]$
		and
			$[x_0, x_0 + 1.1\sqrt{t} / (k - 1)]$.
		By our choice of $t$, we have
			$t \le p(k - 1) < t + 1$,
		and thus,
			$x_0 \le p < x_0 + 1 / (k - 1)$.
		This implies that
		\[
			l
		=	\min(p + \epsilon, x_0 + 1.1\sqrt{t} / (k - 1)) - p
		\ge \min(\epsilon, \sqrt{t} / [10(k - 1)]).
		\]

		Therefore, we conclude that
		\[
			f(p) - f(p + \epsilon)
		\ge	l \cdot \frac{k}{21\sqrt{t}}
		\ge	\frac{1}{21}\min\left(\frac{\epsilon k}{\sqrt{t}}, \frac{1}{10}\right)
		\ge	\frac{1}{21}\min\left(\epsilon\sqrt{2k}, \frac{1}{10}\right)
		\ge	\frac{\epsilon\sqrt{k}}{15}.
		\]
		Here the last step holds since
		our assumption $\epsilon < 1/(15\sqrt{k})$
		implies that $\epsilon\sqrt{2k} < 1/10$.
	\end{proof}

%% file: missing.tex
\section{Missing Proofs from Section~\ref{sec:const-k}}
	\subsection{Proof of Claim~\ref{claim:fnt}}
		\begin{proof}[Proof of Claim~\ref{claim:fnt}]
			By assumption, the frequency tensor $A$ can be written as
				$A = (1 - \epsilon)\hat{P} + \epsilon N$,
			where $\hat{P}$ and $N$ are probability $n^t$-tensors.
			In particular, $\hat{P}$ denotes the frequency
			among the $m$ ``good'' samples drawn from $\tp{p}{t}$,
			while $N$ denotes the frequency among the other batches.
			We prove the lemma by induction on $t$.

			\paragraph{Base case.}
			When $t = 1$, we have $\distset(n, t, A) = \{A\}$, and thus,
			\[
				\min_{q\in\distset(n, t, A)}\Delta(p, q)
			=	\Delta(p, A)
			\le	\Delta(p, \hat{P}) + \Delta(\hat{P}, A)
			=	\Delta(p, \hat{P}) + \epsilon\Delta(\hat{P}, N)
			\le \Delta(p, \hat{P}) + \epsilon.
			\]
			It is well-known that with probability $1 - \delta_0$,
				\[\Delta(p, \hat{P}) \le C\sqrt\frac{n + \ln\delta_0^{-1}}{m}\]
			for some absolute constant $C$,
			and thus,
			\[
				\min_{q\in\distset(n, t, A)}\Delta(p, q)
			\le C\sqrt\frac{n + \ln\delta_0^{-1}}{m} + \epsilon
			\le f_{n, 1}(\epsilon, m).
			\]

			\paragraph{Inductive step.}
			Let $\hat{p}$ and $\delta$ be the marginals of $\hat{P}$ and $N$, respectively.
			Then the marginal of $A$ is given by
				$a = (1 - \epsilon)\hat p + \epsilon \delta$.
			Moreover, the $i$-th slice of $A$ after normalization is given by
			\[
				A_i
			=	\frac{(1 - \epsilon)\hat{p}_i\cdot\hat{P}_i + \epsilon\delta_i\cdot N_i}{(1 - \epsilon)\hat{p}_i + \epsilon\delta_i}
			=	(1 - \epsilon'_i)\hat{P}_i + \epsilon'_i N_i,
			\]
			where
			\[
				\epsilon'_i 
			=	\frac{\epsilon\delta_i}{(1 - \epsilon)\hat{p}_i + \epsilon\delta_i}.
			\]
			Moreover, $A_i$ contains $m'_i = m\cdot\hat{p}_i$ samples drawn from $\tp{p}{(t-1)}$.

			Let
				$\alpha = (1 - \sqrt[t]{1 - \epsilon})/\epsilon$
			and
				$\beta = 1 / t$.
			We consider the following two cases.

			\paragraph{Case 1: For every $i \in [n]$, either $\delta_i / \hat{p}_i \ge 1 - \alpha$ or $\hat{p}_i \le \beta / n$.}
			Since $a \in \distset(n, t, A)$, we have
			\[
				\min_{q\in\distset(n, t, A)}\Delta(p, q)
			\le \Delta(p, a)
			\le \Delta(p, \hat{p}) + \Delta(\hat{p}, a)
			=	\Delta(p, \hat{p}) + \epsilon\Delta(\hat{p}, \delta).
			\]
			Moreover, since for any $i\in[n]$, either $\delta_i \ge (1 - \alpha)\hat{p}_i$ or $\hat{p}_i\le\beta / n$, it holds that
				$\hat{p}_i - \delta_i \le \alpha\hat{p}_i + \beta / n$,
			and thus
			\[
				\Delta(\hat{p}, \delta)
			=	\sum_{i = 1}^{n}\max(\hat{p}_i - \delta_i, 0)
			\le \sum_{i = 1}^{n}(\alpha\hat{p}_i + \beta / n)
			= \alpha + \beta.
			\]
			It follows that
			\[
				\min_{q\in\distset(n, t, A)}\Delta(p, q)
			\le \Delta(p, \hat{p}) + \epsilon(\alpha + \beta)
			= \Delta(p, \hat{p}) + 1 - \sqrt[t]{1 - \epsilon} + \epsilon / t
			\le \frac{3\epsilon}{t} + \Delta(p, \hat{p}).
			\]
			Here the last step $\sqrt[t]{1 - \epsilon} \ge 1 - 2\epsilon / t$ follows from Lemma~\ref{lem:1-eps-add}.

			\paragraph{Case 2: For some $i \in [n]$, both $\delta_i / \hat{p}_i < 1 - \alpha$ and $\hat{p}_i > \beta / n$ hold.}
			In this case, we have
			\[
				\epsilon'_i 
			=	\frac{\epsilon(\delta_i/\hat{p}_i)}{1 - \epsilon + \epsilon(\delta_i/\hat{p}_i)}
			\le \frac{\epsilon(1 - \alpha)}{1 - \epsilon\alpha}
			=	1 - (1 - \epsilon)^{(t - 1) / t}
			\]
			and
				\[m'_i = m\cdot\hat{p}_i \ge \beta m / n = m/(nt).\]
			We have the following bound:
			\[
				\min_{q\in\distset(n, t, A)}\Delta(p, q)
			\le \min_{q\in\distset(n, t-1, A_i)}\Delta(p, q).
			\]

			Combining the two cases shows that $\min_{q\in\distset(n, t, A)}\Delta(p, q)$
			is upper bounded by either
				$3\epsilon / t + \Delta(p, \hat{p})$
			or
				$\min_{q\in\distset(n, t-1, A_i)}\Delta(p, q)$
			for some $i \in [n]$ such that
				$\epsilon'_i \le 1 - (1 - \epsilon)^{(t - 1) / t}$
			and
				$m'_i \ge m/(nt)$.
			According to the induction hypothesis,
			for each $i \in [n]$,
			it holds with probability $1 - (n + 1)^{t - 2}\cdot\delta_0$ that
			\[
				\min_{q\in\distset(n, t-1, A_i)}\Delta(p, q)
			\le f_{n, t-1}\left(\epsilon'_i, m'_i\right).
			\]
			Moreover, with probability $1 - \delta_0$,
				\[\Delta(p, \hat{p}) \le C\sqrt\frac{n + \ln\delta_0^{-1}}{m}.\]
			By a union bound, with probability
				\[1 - n\cdot(n+1)^{t-2}\delta_0-\delta_0 \ge 1 - (n+1)^{t-1}\delta_0,\]
			it holds that
			\[
				\min_{q\in\distset(n, t, A)}\Delta(p, q)
			\le	\max\left(\frac{3\epsilon}{t} + C\sqrt{\frac{n + \ln\delta_0^{-1}}{m}},
					f_{n, t - 1}\left(1 - (1 - \epsilon)^{(t-1)/t}, \frac{m}{nt}\right)\right)
			=	f_{n, t}(\epsilon, m).
			\]
			This completes the inductive step.
		\end{proof}

	\subsection{Proof of Claim~\ref{claim:fnt-bound}}
		\begin{proof}[Proof of Claim~\ref{claim:fnt-bound}]
			Define $(\epsilon_k, m_k) = (\epsilon, m)$ and
			\[
				(\epsilon_{t-1}, m_{t-1}) = \left(1 - (1 - \epsilon_t)^{(t-1) / t}, m_t/(nt)\right).
			\]
			Then by Equations \eqref{eq:recur-1}~and~\eqref{eq:recur-t},
			\begin{align*}
				f_{n, k}(\epsilon_k, m_k)
			=	&\max\left(
				\frac{3\epsilon_k}{k} + C\sqrt{\frac{n + \ln\delta_0^{-1}}{m_k}},
				f_{n, k - 1}(\epsilon_{k-1}, m_{k-1})
			\right)\\
			=	&\max\left(
				\frac{3\epsilon_k}{k} + C\sqrt{\frac{n + \ln\delta_0^{-1}}{m_k}},
				\frac{3\epsilon_{k-1}}{k-1} + C\sqrt{\frac{n + \ln\delta_0^{-1}}{m_{k-1}}},
				f_{n, k-2}(\epsilon_{k-2}, m_{k-2})
			\right)\\
			=	&\cdots\\
			=	&\max_{t\in[k]}\left(
				\frac{3\epsilon_t}{t} + C\sqrt{\frac{n + \ln\delta_0^{-1}}{m_t}}
			\right).
			\end{align*}
			Moreover, by Lemma~\ref{lem:1-eps-add} and a simple induction,
				$\epsilon_t = 1 - (1 - \epsilon)^{t / k} \le 2t\epsilon/k$
			and
				$m_t = m/(n^{k - t}\arrange{k}{k - t})$,
			where $\arrange{n}{m}$ denotes $n(n - 1)\cdots(n - m + 1)$.
			It follows that for any $t \in [k]$,
			\[
				\frac{3\epsilon_t}{t} + C\sqrt{\frac{n + \ln\delta_0^{-1}}{m_t}}
			\le \frac{3}{t} \cdot \frac{2t\epsilon}{k} + C\sqrt\frac{n^{k-t}\arrange{k}{k-t}(n + \ln\delta_0^{-1})}{m}
			\le \frac{6\epsilon}{k} + C\sqrt\frac{n^k\cdot k!\cdot(n + \ln\delta_0^{-1})}{m}.
			\]
			\end{proof}